%% file: main.tex
\documentclass[runningheads]{llncs}
\usepackage[T1]{fontenc}
\usepackage{graphicx}
\usepackage{booktabs}
\usepackage[misc]{ifsym}
\newcommand{\corr}{(\Letter)}
\usepackage{amsmath,amssymb}
\usepackage{xspace}
\usepackage{enumitem}
\usepackage{multirow}
\usepackage{xcolor}
\usepackage{colortbl}
\usepackage[hidelinks]{hyperref}
\usepackage{cite}
\usepackage{algorithm}
\usepackage[noend]{algorithmic}

\graphicspath{{./}}

\newcommand{\slc}{\textsc{SLC}\xspace}
\newcommand{\pp}{\,\text{pp}}
\DeclareMathOperator{\sigmoid}{\sigma}
\DeclareMathOperator{\logit}{logit}
\DeclareMathOperator{\AUC}{AUC}

\DeclareMathOperator{\Ber}{Ber}
\DeclareMathOperator{\Var}{Var}

\newcommand{\base}{\textsf{Base}\xspace}
\newcommand{\platt}{\textsf{Platt}\xspace}
\newcommand{\plattt}{\textsf{Platt-T}\xspace}
\newcommand{\iso}{\textsf{Iso}\xspace}
\newcommand{\hist}{\textsf{Hist}\xspace}
\newcommand{\rescal}{\textsf{ResCal}\xspace}
\newcommand{\rescaliso}{\textsf{ResCal{+}Iso}\xspace}
\newcommand{\naive}{\textsf{Naive}\xspace}
\newcommand{\slct}{\textsc{SLC}\textsf{-T}\xspace}
\newcommand{\slciso}{\textsc{SLC}\textsf{+Iso}\xspace}
\newcommand{\SD}[1]{{\scriptsize$\,\pm\,$#1}}

\begin{document}

\title{Recovering Stranded Discrimination in Knowledge Tracing:\\Per-Item Bias Correction via Empirical-Bayes Shrinkage}
\titlerunning{Recovering Stranded Discrimination in Knowledge Tracing}

\toctitle{Recovering Stranded Discrimination in Knowledge Tracing: Per-Item Bias Correction via Empirical-Bayes Shrinkage}
\tocauthor{Xiaoran Yan, Cheng Tang, Atsushi Shimada}

\author{Xiaoran Yan \and
Cheng Tang \and
Atsushi Shimada \corr}
\authorrunning{X. Yan et al.}
\institute{Kyushu University, Fukuoka, Japan\\
\email{xiaoran.y@outlook.com}, \email{tang@limu.ait.kyushu-u.ac.jp},\\
\email{atsushi@ait.kyushu-u.ac.jp}}

\maketitle

\begin{abstract}
Deployed knowledge-tracing models are typically frozen after training, yet systematic per-item logit bias arises---from limited per-item expressivity in backbone architectures and from post-deployment shifts in item properties---degrading prediction quality.
Global post-hoc calibrators such as Platt scaling, temperature scaling, and isotonic regression improve probability estimates but leave discriminative ability, as measured by AUC, unchanged.
This AUC invariance is a structural consequence of monotone score-only transforms; recovering the stranded discrimination requires conditioning on item identity.
We propose \slc (State-space Logit Correction), which converts binary observations to Gaussian pseudo-observations via Laplace/IRLS, applies empirical-Bayes shrinkage through a Kalman smoother, and fits an offset-Platt link.
The state-space formulation also yields a detectability bound that characterizes the Bernoulli information floor, explaining why temporal tracking provides no benefit at current data densities.
Across four datasets, five backbones, and three seeds, \slc improves AUC on all four datasets and NLL on three, with the advantage concentrating on sparse items.
Cross-domain controls suggest that the same phenomenon can arise beyond education when the deployed backbone leaves entity-level bias.

\end{abstract}

\section{Introduction}
\label{sec:intro}

Knowledge-tracing (KT) models~\cite{corbett1995kt} estimate the probability that a student will answer an item correctly; these probabilities drive adaptive item selection, mastery gating, and early-warning systems.
Per-item logit bias arises from two sources: backbone architectures with limited per-item expressivity produce structural prediction errors, and post-deployment shifts in item properties~\cite{lee2023kt_longitudinal,lee2025concept_drift_kt}---difficulty changes, new items, population evolution---further compound them.

The standard response is post-hoc calibration: Platt scaling~\cite{platt1999probabilistic}, temperature scaling~\cite{guo2017calibration}, or isotonic regression~\cite{zadrozny2002transforming}.
These methods improve probability estimates, yet the model's \emph{discriminative} ability---its AUC---remains unchanged because global score-only transforms adjust the \emph{scale} of predictions but not their \emph{ordering}.
We draw a sharp distinction: \textbf{calibration} adjusts probability scale (AUC-invariant); \textbf{correction} recovers ranking quality (AUC-improving). This paper addresses correction.

Figure~\ref{fig:symptom} illustrates this on ASSISTments~2017 (temporal split, 5~backbones): raw and Platt-scaled AUC are identical at every time slice, yet per-item residual correction (\rescal) recovers $+2.6\pp$ of hidden headroom (Eedi: $+3.5\pp$); \slc further improves on this via shrinkage.

\begin{figure}[t]
  \centering
  \includegraphics[width=0.95\textwidth]{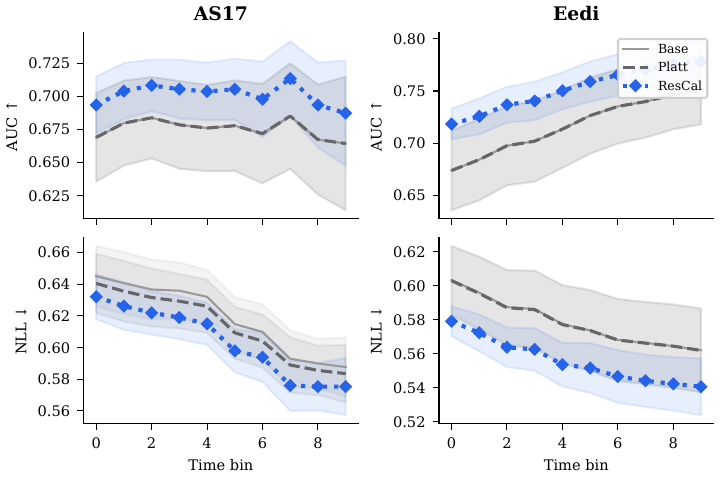}
  \caption{ASSISTments~2017, temporal split, 5~backbones averaged. \textbf{Left:} AUC over time for the raw backbone (\base) and Platt scaling (\platt)---the curves are identical, confirming AUC invariance of global calibration. \textbf{Right:} per-item residual correction (\rescal) recovers $+2.6\pp$ of hidden AUC headroom that global calibration structurally cannot access.}
  \label{fig:symptom}
\end{figure}

This behavior is structural: Lemma~\ref{lem:auc} shows that any strictly monotone, score-only transformation preserves rankings and hence AUC; recovering headroom therefore requires conditioning on item identity.
We propose \slc (State-space Logit Correction)\footnote{Code: \url{https://github.com/xiaoran-y/SLC}}, which models per-item bias as a Gaussian random effect~\cite{efron1973stein}, converts binary observations to Gaussian pseudo-observations via Laplace/IRLS, and pools them with a Kalman smoother~\cite{durbin2012time,fahrmeir1992kalman}.
The state-space formulation also yields a detectability bound (Proposition~\ref{prop:detect}) that quantifies when temporal tracking becomes viable.
The corrected prediction takes the form
\begin{equation}\label{eq:model}
  p = \sigmoid\!\big(a\,\eta_0 + b_0 + \hat{b}_i\big),
\end{equation}
where $\eta_0=\logit(p_0)$ is the frozen backbone logit and $(a,b_0)$ are global affine parameters.
The framework extends to temporal tracking $\beta(i,t)=b_i+u_i(t)$, but Proposition~\ref{prop:detect} shows that the minimum detectable drift far exceeds observed temporal variation at current KT data densities.

Our contributions are:
\begin{enumerate}[leftmargin=*,nosep]
\item \textbf{Stranded discrimination.} We identify per-item ``stranded'' AUC headroom in deployed KT models. AUC invariance of monotone score-only transforms (Lemma~\ref{lem:auc}) serves as a diagnostic; a five-level baseline ladder confirms that neither score-only nor time-only conditioning recovers this headroom across all 20~configurations.

\item \textbf{Per-item shrinkage pipeline.} We propose \slc: Laplace/IRLS pseudo-observations, empirical-Bayes shrinkage via Kalman smoothing, and an offset-Platt link. The additive per-item form is theoretically motivated (Proposition~\ref{prop:additive}); a detectability bound (Proposition~\ref{prop:detect}) explains why temporal tracking is information-limited at current densities and predicts the viability threshold (on the order of $10^5$ obs/item).

\item \textbf{Comprehensive evaluation.} Four KT datasets, five backbones, three seeds; density-stratified analysis, calibration-fraction sweep, synthetic regime map, and cross-domain controls including a non-KT flight-delay experiment.
\end{enumerate}

\section{Related Work}
\label{sec:related}

\subsection{Knowledge-Tracing Models}

DKT~\cite{piech2015dkt} applies recurrent networks; SAKT~\cite{pandey2019sakt} and AKT~\cite{ghosh2020akt} use self-attention; DKVMN~\cite{zhang2017dkvmn} augments memory networks; LPKT~\cite{shen2021lpkt} models the learning process explicitly.
These architectures differ in per-item expressivity: DKT, SAKT, and DKVMN operate at the skill level and share representations across items within a skill, while AKT and LPKT include per-item parameters.
Even with per-item modeling, frozen backbones accumulate residual per-item bias after deployment.
\slc is post-hoc and backbone-agnostic: it corrects this residual bias from any frozen model's logits without retraining.

\subsection{Post-Hoc Calibration}

Post-hoc calibration adjusts predicted probabilities to match observed frequencies.
Platt scaling~\cite{platt1999probabilistic}, temperature scaling~\cite{guo2017calibration}, isotonic regression~\cite{zadrozny2002transforming}, and histogram binning~\cite{zadrozny2001obtaining} are all score-only transforms; strictly monotone variants leave AUC invariant (Lemma~\ref{lem:auc}).
ECE is not a proper scoring rule~\cite{gruber2022better}; we treat NLL as a co-primary metric and ECE as a diagnostic.

\subsection{Per-Group and Per-Instance Calibration}

Several works condition calibration on input features, including class-wise~\cite{frenkel2021cts} and parameterized~\cite{tomani2022pts} temperature scaling, density-aware calibration~\cite{tomani2023dac}, and field-aware calibrators~\cite{pan2020field}.
\slc operates in a metadata-only regime (item~id + time index, no learned embeddings).
Class-wise scaling with $K{\gg}1$ categories reduces to unregularized per-item estimation (our \naive baseline); \slc adds shrinkage.
In the static limit, \slc reduces to ridge logistic regression with per-item intercepts~\cite{breslow1993glmm}.
Logit adjustment~\cite{menon2021logitadj} shares the per-class offset idea but targets class imbalance.

\subsection{Temporal Adaptation and State-Space Models}

Test-time adaptation (e.g., Tent~\cite{wang2021tent}) modifies model parameters online, while dynamic IRT models~\cite{wang2013dynamic_irt,vtirt2023} jointly re-estimate ability and difficulty; both require either model access or full re-estimation.
\slc instead treats the backbone as frozen and adopts the Laplace/IRLS + Kalman inference techniques developed for state-space GLMMs~\cite{durbin2012time,fahrmeir1992kalman,breslow1993glmm} as a post-hoc per-item correction algorithm, with the resulting shrinkage paralleling James--Stein estimation~\cite{efron1973stein,efron2012large_scale}.

\section{Method}
\label{sec:method}

\subsection{Problem Setting}
\label{sec:setting}

A frozen KT backbone produces logits $\eta_0(x)=\logit(p_0(x))$ for each interaction $x=(s,i,t)$ (student $s$, item $i$, time index $t$).
Data is partitioned temporally: train $\to$ calibration $\to$ test (strictly later), exposing genuine drift.
The post-hoc correction uses $\eta_0$, labels $y\in\{0,1\}$, item IDs, and time indices from the calibration window only---no test labels, no backbone parameter updates.

\subsection{AUC Invariance of Score-Only Calibration}
\label{sec:lemma1}

The central structural observation is that global calibration is inherently unable to improve AUC.
This invariance is a classical fact~\cite{hanley1982roc,fawcett2006roc}; we restate it because it serves as the diagnostic for stranded headroom:

\begin{lemma}[AUC invariance]\label{lem:auc}
Let $s(x)\in\mathbb{R}$ be a scalar score and $\phi:\mathbb{R}\to\mathbb{R}$ be strictly increasing. Then $\AUC(\phi(s))=\AUC(s)$.
\end{lemma}

\begin{proof}
AUC equals the probability that a randomly drawn positive receives a higher score than a randomly drawn negative~\cite{hanley1982roc}.
Since $\phi$ is strictly increasing, $s(x_+)>s(x_-)$ iff $\phi(s(x_+))>\phi(s(x_-))$.
All pairwise orderings are preserved, and AUC---which depends only on these orderings---remains unchanged (see also~\cite{fawcett2006roc}).
\end{proof}

Platt scaling ($a>0$) and temperature scaling ($T>0$) satisfy strict monotonicity and are exactly AUC-invariant.
Isotonic regression and histogram binning are non-decreasing but piecewise constant, so they fall outside Lemma~\ref{lem:auc} and can in principle change AUC through ties; in practice the effect is negligible with continuous logits.
\plattt (score + time, no item id) changes AUC only marginally in our experiments: without item identity, per-item heterogeneity cannot be resolved.
The implication: \emph{recovering AUC headroom requires conditioning on at least the item identity.}

\subsection{Per-Item Correction Form}
\label{sec:prop1}

The following proposition is a standard conditional-mean projection argument; we state it to fix the form of the correction.

\begin{proposition}\label{prop:additive}
Let $\eta^*(s,i,t)$ denote the Bayes-optimal logit for student $s$, item $i$, time $t$, and let $\eta_0(s,i,t)$ be the frozen backbone logit.
Define the item-specific bias $\beta(i) = \mathbb{E}[\eta^* - \eta_0 \mid i]$ and the residual
\begin{equation}\label{eq:additive}
  \epsilon(s,i,t) = \eta^*(s,i,t) - \eta_0(s,i,t) - \beta(i),
\end{equation}
so that $\mathbb{E}[\epsilon \mid i] = 0$ by construction.
Then the MSE-optimal additive correction of $\eta_0$ that depends only on item identity is
\[
  \hat{\eta}(s,i,t) = \eta_0(s,i,t) + \beta(i).
\]
\end{proposition}

\begin{proof}
Among all corrections of the form $\eta_0 + f(i)$, the MSE $\mathbb{E}[(\eta^* - \eta_0 - f(i))^2]$ is minimized by $f(i) = \mathbb{E}[\eta^* - \eta_0 \mid i] = \beta(i)$.
\end{proof}

Because $\beta(i)$ is defined as the conditional mean, $\mathbb{E}[\epsilon \mid i] = 0$ holds by construction.
How much AUC headroom per-item correction can recover depends on $\mathrm{Var}(\beta(i))$, the between-item component of the backbone error, which we evaluate empirically in Section~\ref{sec:experiments}.

A standard 1PL IRT model of item difficulty drift provides an idealized setting to verify the correction form.
Extending the conditioning to both item and time ($\beta(i,t){=}\mathbb{E}[\eta^*{-}\eta_0\mid i,t]$), the residual vanishes entirely:

\begin{corollary}[1PL exactness]\label{prop:irt}
Let $Y \sim \Ber(\sigmoid(\theta_s - b_i(t)))$ with $b_i(t) = b_i^{\mathrm{train}} + \Delta_i(t)$, and let the frozen backbone produce $\eta_0(s,i) = \alpha(\theta_s - b_i^{\mathrm{train}}) + c_0$ for some backbone scale factor $\alpha > 0$ and shift $c_0$ (distinct from the residual term in Proposition~\ref{prop:additive}).
Then the Bayes-optimal logit $\eta^* = \theta_s - b_i(t)$ is \textbf{exactly} recovered by the offset-Platt + per-item correction:
\[
  \eta^* = a\,\eta_0 + b + \beta(i,t), \qquad a = 1/\alpha,\;\; b = -c_0/\alpha,\;\; \beta(i,t) = -\Delta_i(t).
\]
The residual MSE is zero.
\end{corollary}

\begin{proof}
$\eta^* = \theta_s - b_i^{\mathrm{train}} - \Delta_i(t) = \frac{\eta_0 - c_0}{\alpha} - \Delta_i(t) = \frac{1}{\alpha}\,\eta_0 - \frac{c_0}{\alpha} - \Delta_i(t) = a\,\eta_0 + b + \beta(i,t).$
\end{proof}

Under 1PL drift, the exact correction $\beta(i,t)$ is student-free and additive (not scaled by $a$); the static $\hat{b}_i$ averages over time.
Synthetic experiments (Section~\ref{sec:regime}) verify robustness beyond the 1PL assumption.

\subsection{Model Specification}
\label{sec:model}

We formalize the per-item bias as a Gaussian random effect within a generalized linear mixed model (GLMM):
\begin{equation}\label{eq:ssglmm}
  y \sim \Ber\!\big(\sigmoid(a\,\eta_0 + b_0 + b_i)\big), \qquad b_i \sim \mathcal{N}(0, \sigma_b^2).
\end{equation}

Here $(a, b_0)$ absorb global scale/shift distortion, $b_i$ captures per-item logit shift, and $\mathcal{N}(0, \sigma_b^2)$ provides shrinkage.
Critically, $b_i$ enters as an additive offset, not scaled by $a$, matching Corollary~\ref{prop:irt}; the alternative form $\sigmoid(a(\eta_0+\hat{b}_i)+b_0)$ performs worse empirically (Section~\ref{sec:ablations}).
\slc fits this parameterization in two stages: $\{b_i\}$ are first estimated under the working model $p{=}\sigma(\eta_0 {+} b_i)$ with implicit $(a,b_0){=}(1,0)$; $(a, b_0)$ are then fit treating the $\hat{b}_i$ as fixed offsets (Section~\ref{sec:link}).
Since $b_i$ is additive and independent of $a$ (Corollary~\ref{prop:irt}), the two stages do not interfere; this separation also preserves the diagonal Hessian that enables $O(N{+}K)$ per-item estimation (appendix).

The GLMM also admits a state-space extension: $\beta(i,t) = b_i + u_i(t)$, $u_i(t) = u_i(t{-}1) + \varepsilon_t$, $\varepsilon_t \sim \mathcal{N}(0, \sigma_u^2)$.
However, temporal tracking does not improve AUC or NLL on our datasets (Section~\ref{sec:ablations}, Proposition~\ref{prop:detect}); the default \slc uses static $b_i$.

\subsection{Estimation via Kalman Smoother}
\label{sec:estimation}

The binary likelihood does not admit conjugate Kalman updates, so we linearize it with the standard Laplace approximation~\cite{durbin2012time}.
Time bins are equal-count (quantile) partitions of cumulative interaction indices from the train and calibration windows; for the static model the final estimate pools across all bins and is invariant to the bin definition.
For each (item, time-bin) cell $(i,t)$, given a current estimate $\hat{b}_i^{(\text{prev})}$, we compute predicted probabilities $p_n = \sigmoid(\eta_0(x_n) + \hat{b}_i^{(\text{prev})})$ for each observation $n$ in the cell and then form
\begin{align}
  W_{i,t} &= \sum_{n \in (i,t)} p_n(1-p_n), \label{eq:weight} \\
  z_{i,t} &= \hat{b}_i^{(\text{prev})} + \frac{\sum_{n \in (i,t)} (y_n - p_n)}{W_{i,t}}. \label{eq:pseudo}
\end{align}
These yield approximate Gaussian observations $z_{i,t} \sim \mathcal{N}(b_i, 1/W_{i,t})$, where $W_{i,t}$ is the Fisher information weight (effective sample size) and $z_{i,t}$ the pseudo-residual, bridging the binary likelihood to a Gaussian state-space model.
In the static case, the Kalman smoother reduces to a weighted-average shrinkage estimator:
\begin{equation}\label{eq:static}
  \hat{b}_i = \frac{\sum_t W_{i,t}\, z_{i,t}}{1/\sigma_b^2 + \sum_t W_{i,t}}.
\end{equation}
The shrinkage fraction $\lambda_i = \sum_t W_{i,t} / (1/\sigma_b^2 + \sum_t W_{i,t})$ ranges from near~0 for sparse items (estimate stays close to the prior mean) to near~1 for dense items (estimate tracks the empirical residual).
This is an empirical-Bayes shrinkage estimator~\cite{efron2012large_scale} that parallels James--Stein shrinkage~\cite{efron1973stein}: under the Laplace--normal proxy the estimator dominates the naive per-item mean in MSE whenever $K\ge 3$; in the original Bernoulli model the dominance is approximate, but the empirical gains on sparse items are substantial (Section~\ref{sec:ablations}).
The prior variance $\sigma_b^2 = 1.0$ is fixed throughout; the estimator is insensitive to this choice because $\lambda_i$ saturates rapidly for $W_i \gg 1/\sigma_b^2$ (sweeping $\sigma_b^2 \in [0.1, 10]$ on AS17 changes AUC by $<0.4\pp$).
In the static case, Eq.~\eqref{eq:static} is the MAP estimate of $\ell_2$-penalized logistic regression with per-item intercepts and the backbone logit as offset; the Hessian is diagonal because each observation involves exactly one item, so the Newton step decomposes into $K$ independent scalar updates (appendix).
Our deployed \slc uses a single step from $b_i^{(0)}{=}0$ (Algorithm~\ref{alg:slc}); the offset-Platt fit (Section~\ref{sec:link}) absorbs residual global bias.

Items with zero calibration observations receive $\hat{b}_i = 0$ (the prior mean), defaulting to \platt; for sparse items, $\lambda_i \to 0$ pulls the estimate toward the prior without any numerical floor.
For the temporal extension, Rauch--Tung--Striebel smoothing applies at cost $O(K \cdot T)$.

\subsection{Link Estimation}
\label{sec:link}

Given $\{\hat{b}_i\}$, we fit $(a, b_0)$ treating $\hat{b}_i$ as a fixed offset:
\begin{align*}
  (a^*, b_0^*) &= \arg\min_{a,b_0} \sum_n -y_n \log p_n - (1{-}y_n)\log(1{-}p_n), \\
  p_n &= \sigmoid(a\,\eta_0(x_n) + b_0 + \hat{b}_{i(n)}).
\end{align*}
This is a convex logistic regression with two parameters, solved by IRLS in a few iterations.
The offset-Platt link correctly treats $\hat{b}_i$ as a random effect that should not be rescaled, matching the GLMM parameterization (Eq.~\eqref{eq:ssglmm}) and Corollary~\ref{prop:irt}.

Alternatively, we can fit a monotone function $g$ via isotonic regression on the corrected logit $\tilde{\eta}_n = \eta_0(x_n) + \hat{b}_{i(n)}$.
This preserves rankings (hence AUC) but offers more flexibility than affine mapping, at the cost of higher variance on sparse data.
Offset-Platt is the recommended default.

\subsection{Temporal Drift Detectability Bound}
\label{sec:prop2}

We explain this via a statistical detectability bound:

\begin{proposition}[Detectability bound]\label{prop:detect}
Consider item $i$ observed across $T$ time bins.
Let $n_t = |\mathcal{C}_{i,t}|$ denote the number of observations in bin $t$, and let $\beta_t$ denote the true per-item logit bias in bin $t$.
Define the drift increment $\delta_t = \beta_t - \beta_{t-1}$.
Under the Laplace approximation (Section~\ref{sec:estimation}), the pseudo-observation for bin $t$ satisfies
\[
  z_{i,t} \;\sim\; \mathcal{N}(\beta_t,\; 1/W_{i,t}), \qquad W_{i,t} = \sum_{n \in \mathcal{C}_{i,t}} p_n(1-p_n).
\]
\begin{enumerate}[leftmargin=*,nosep,label=\textbf{(\roman*)}]
\item The Fisher weight is bounded: $W_{i,t} \le n_t / 4$, since $p(1-p) \le 1/4$ for all $p \in [0,1]$.
\item Treating adjacent bins as conditionally independent under the same approximation, the variance of the difference $z_{i,t} - z_{i,t-1}$ (which estimates $\delta_t$) is
\[
  \Var(z_{i,t} - z_{i,t-1}) = \frac{1}{W_{i,t}} + \frac{1}{W_{i,t-1}} \;\ge\; \frac{4}{n_t} + \frac{4}{n_{t-1}} \;\ge\; \frac{8}{n_{\min}},
\]
where $n_{\min} = \min(n_t, n_{t-1})$.
\item A Wald test at significance level $\alpha$ detects $\delta_t \neq 0$ only if
\begin{equation}\label{eq:detect}
  |\delta_t| \;\ge\; \delta_{\min} \;=\; z_\alpha \sqrt{\frac{1}{W_{i,t}} + \frac{1}{W_{i,t-1}}} \;\ge\; z_\alpha \sqrt{\frac{8}{n_{\min}}}.
\end{equation}
\end{enumerate}
\end{proposition}

\begin{proof}\leavevmode
\begin{enumerate}[leftmargin=*,nosep,label=(\roman*)]
\item $W_{i,t} = \sum_{n} p_n(1{-}p_n) \le \sum_{n} \tfrac{1}{4} = n_t/4$, with equality when all $p_n = 1/2$.
\item Under this approximation, we treat $z_{i,t}$ and $z_{i,t-1}$ as conditionally independent because they are computed from disjoint observation sets, so $\Var(z_{i,t} - z_{i,t-1}) = 1/W_{i,t} + 1/W_{i,t-1}$. Applying~(i) and then $1/W \ge 4/n$ yields the bound.
\item The Wald statistic
\[
  Z = (z_{i,t} - z_{i,t-1})\big/\!\sqrt{1/W_{i,t} + 1/W_{i,t-1}}
\]
is $\mathcal{N}(0,1)$ under $H_0{:}\,\delta_t{=}0$. Rejection at level $\alpha$ requires $|Z| \ge z_\alpha$, i.e.\ $|\delta_t| \ge z_\alpha\sqrt{1/W_{i,t} + 1/W_{i,t-1}}$.
\end{enumerate}
\end{proof}

On ASSISTments~2017 (strongest drift), the median item has 48 observations across 10 bins (${\sim}4.8$/bin, $W_{\max}\le 1.20$).
At $\alpha=0.05$: $\delta_{\min} \approx 2.53$~logit, i.e., an item's correct rate would need to shift from 50\% to 92\% within a single bin for the change to be detectable.
Table~\ref{tab:detect} reports all four datasets.
As an independent check, we compute the median observed adjacent-bin logit change $|\hat\delta_{\text{obs}}|$ from empirical per-item correct rates; this is an \emph{upper bound} on true drift because it includes estimation noise.
On every dataset, $|\hat\delta_{\text{obs}}| \ll \delta_{\min}$, confirming that temporal tracking operates deep in the noise-dominated regime.
Reducing $T$ lowers $\delta_{\min}$ (as $\sqrt{T}$), but also concentrates the signal into fewer comparisons with higher per-bin variance; at the extreme $T{=}2$ on AS17, $\delta_{\min}$ drops to ${\approx}1.13$~logit (50\%$\to$75\% correct rate), still implausible.
The Kalman smoother implicitly optimizes this resolution--variance trade-off; its uniform loss to the static estimator in all 20 configurations (Section~\ref{sec:ablations}) confirms that no temporal granularity recovers useful signal at these densities.
Rearranging Eq.~\eqref{eq:detect} suggests a rough viability threshold on the order of $10^5$ total obs/item for a moderate drift of $\delta = 0.05$~logit/bin.

\subsection{Algorithm Summary}
\label{sec:algo}

Algorithm~\ref{alg:slc} summarizes the complete \slc pipeline.
Lines~2--6 convert the binary observations in each (item, time-bin) cell into Gaussian pseudo-observations $(z_{i,t}, W_{i,t})$ via the Laplace approximation of Section~\ref{sec:estimation}; lines~7--9 apply the empirical-Bayes shrinkage of Eq.~\eqref{eq:static} to obtain the per-item offsets $\hat{b}_i$; line~10 fits the offset-Platt link of Section~\ref{sec:link}.
The only hyperparameter is the prior variance $\sigma_b^2$, fixed to $1.0$ throughout.

\begin{algorithm}[t]
\caption{\slc (State-space Logit Correction)}\label{alg:slc}
\begin{algorithmic}[1]
\REQUIRE Backbone logits $\{\eta_{0,n}\}_{n=1}^{N}$, calibration labels $\{y_n\}$, item IDs $\{i_n\}$, time indices $\{t_n\}$, prior variance $\sigma_b^2$
\ENSURE Corrected probabilities $\{\hat{p}_n\}$ for new observations
\STATE Partition calibration data into $(i,t)$ cells: $\mathcal{C}_{i,t} \gets \{n : i_n = i,\, t_n = t\}$
\FOR{each cell $(i,t)$ with $|\mathcal{C}_{i,t}| > 0$}
  \STATE $p_n \gets \sigmoid(\eta_{0,n})$ for $n \in \mathcal{C}_{i,t}$
  \STATE $W_{i,t} \gets \sum_{n \in \mathcal{C}_{i,t}} p_n(1 - p_n)$
  \STATE $z_{i,t} \gets \sum_{n \in \mathcal{C}_{i,t}} (y_n - p_n)\, / \, W_{i,t}$
\ENDFOR
\FOR{each item $i \in \{1,\ldots,K\}$}
  \STATE $\hat{b}_i \gets \frac{\sum_t W_{i,t}\, z_{i,t}}{1/\sigma_b^2 + \sum_t W_{i,t}}$
\ENDFOR
\STATE Fit $(a^*, b_0^*) \gets \arg\min_{a,b_0} \sum_n \mathcal{L}_{\text{BCE}}(y_n,\, \sigmoid(a\,\eta_{0,n} + b_0 + \hat{b}_{i_n}))$
\STATE \textbf{Prediction:} $\hat{p}_n = \sigmoid(a^*\,\eta_{0,n} + b_0^* + \hat{b}_{i_n})$
\end{algorithmic}
\smallskip
Complexity is $O(N + K \cdot T)$. On AS17 ($N{\approx}190$k, $K{\approx}3$k), wall-clock time is under 2\,s on a single CPU core. All reported experiments use this single-pass blockwise fit.
\end{algorithm}

\section{Experiments}
\label{sec:experiments}

\subsection{Setup}
\label{sec:setup}

We evaluate on four KT benchmarks spanning a range of drift intensities and data densities (Table~\ref{tab:datasets}).
ASSISTments~2017 (AS17)~\cite{assistments2017dataset} exhibits strong temporal drift with moderate data density; Eedi~\cite{wang2021eedi} exhibits moderate drift with similar density; ASSISTments~2009 (AS09)~\cite{feng2009assistments} has weak drift and extreme sparsity (median 3 observations per item); and Algebra~\cite{stamper2016kddcup} has weak-to-moderate drift with the most extreme sparsity (median 1 observation per item).
Throughout the paper, Algebra denotes a merged dataset constructed from the 2005--2006, 2006--2007, and 2008--2009 releases; we merge these releases to obtain a longer and less fragmented temporal horizon, making drift-trend analysis under strict temporal splits more stable.
All datasets use strict temporal splits with no overlap between training, calibration, and test windows.
This protocol differs from standard KT benchmarks that use random student-level splits, and is essential for exposing genuine temporal drift.
The fraction of test tokens whose item was never observed during calibration (cold-start) is small on AS17 (1.0\%) and Eedi (0.3\%), but substantial on AS09 (12.3\%) and Algebra (24.6\%); for these tokens \slc defaults to the global Platt prediction ($\hat{b}_i{=}0$).

\begin{table}[t]
\caption{Dataset characteristics. Observations per item and per bin are medians computed over the calibration window.}\label{tab:datasets}
\centering
\small
\begin{tabular}{lccccl}
\toprule
Dataset & Drift & obs/item & time bins & obs/bin & Regime \\
\midrule
AS17 & Strong & 48 & 10 & $\sim$4.8 & Clear per-item benefit \\
Eedi & Moderate & 54 & 10 & $\sim$5.4 & Clear per-item benefit \\
AS09 & Weak & $\sim$3 & 5 & $\sim$0.6 & Extreme sparsity \\
Algebra (merged) & Weak--Mod & $\sim$1 & 5 & $\sim$0.2 & Extreme sparsity \\
\bottomrule
\end{tabular}
\end{table}

We evaluate five KT backbones: AKT, DKT, SAKT, DKVMN, and LPKT.
Together they cover the main KT architecture families~\cite{piech2015dkt,pandey2019sakt,ghosh2020akt,zhang2017dkvmn,shen2021lpkt}.
Each uses three seeds (225, 226, 227); we report mean and standard deviation across seeds.
All experiments use a strict temporal split, train $\to$ calibration $\to$ test, with all hyperparameters selected by train-only rolling backtest and no test labels used at any stage.
Our co-primary metrics are AUC (discrimination) and NLL (proper scoring rule). ECE is reported as a diagnostic, not a primary criterion: it is not a proper scoring rule~\cite{gruber2022better}, and NLL improvements can coexist with worse ECE.
Table~\ref{tab:ladder} organizes the baselines by conditioning structure.

\begin{table}[t]
\caption{Baseline ladder. Each level adds richer conditioning to test which factor drives AUC recovery. The classic score-only calibrators are summarized in the appendix.}\label{tab:ladder}
\centering
\small
\begin{tabular}{clll}
\toprule
Level & Method & Conditioning & Tests \\
\midrule
0 & \base & None & Raw backbone \\
1 & \platt (+ TS/\iso/\hist) & Score only (global) & AUC-invariant? (Lemma~\ref{lem:auc}) \\
2 & \plattt & Score + time (no item id) & Does time-only help? \\
3 & \rescal & Score + item (no smooth) & Per-item helps AUC? \\
3+ & \rescaliso & Score + item + isotonic & Strong per-item baseline \\
4 & \naive & Score + item $\times$ time & Smoothing needed? \\
\textbf{5} & \textbf{\slc} & Score + item + Kalman & \textbf{Smoothed per-item} \\
5\textsuperscript{t} & \slct & + temporal tracking & Temporal helps? \\
5\textsuperscript{n} & \slciso & + isotonic link & Link comparison \\
\bottomrule
\end{tabular}
\end{table}

\platt fits $p=\sigmoid(a\,\eta_0+b)$; \plattt adds a time covariate $p=\sigmoid(a\,\eta_0+c\,t+b)$.
\rescal estimates per-item logit offsets by matching empirical and predicted success rates (with a minimum-count threshold); \rescaliso adds isotonic regression on top.
\naive is a per-(item, time-bin) running-average offset with no smoothing, isolating the value of Kalman shrinkage.

\subsection{Main Results}
\label{sec:main_results}

Tables~\ref{tab:auc} and~\ref{tab:nll} present AUC and NLL results averaged over five backbones and three seeds.

\begin{table}[t]
\caption{AUC ($\uparrow$) averaged over 5~backbones $\times$ 3~seeds ($\pm$~avg seed std). Best per column in \textbf{bold}; per-backbone breakdown in the appendix.}\label{tab:auc}
\centering
\small
\begin{tabular}{l cccc}
\toprule
Method & Algebra & AS17 & Eedi & AS09 \\
\midrule
\base           & 0.8135\SD{0.0004} & 0.6814\SD{0.0010} & 0.7189\SD{0.0001} & 0.6782\SD{0.0069} \\
\platt          & 0.8135\SD{0.0004} & 0.6814\SD{0.0010} & 0.7189\SD{0.0001} & 0.6782\SD{0.0069} \\
\plattt         & 0.8135\SD{0.0004} & 0.6837\SD{0.0009} & 0.7190\SD{0.0001} & 0.6786\SD{0.0065} \\
\rescal         & 0.8243\SD{0.0005} & 0.7065\SD{0.0009} & 0.7534\SD{0.0002} & 0.6784\SD{0.0069} \\
\rescaliso      & 0.8243\SD{0.0005} & 0.7064\SD{0.0009} & 0.7534\SD{0.0002} & 0.6783\SD{0.0068} \\
\naive          & 0.8097\SD{0.0009} & 0.6946\SD{0.0008} & 0.7375\SD{0.0003} & 0.6653\SD{0.0029} \\
\midrule
\textbf{\slc}   & \textbf{0.8325}\SD{0.0007} & \textbf{0.7182}\SD{0.0008} & \textbf{0.7579}\SD{0.0002} & \textbf{0.7066}\SD{0.0049} \\
\slct           & 0.8313\SD{0.0008} & 0.7166\SD{0.0008} & 0.7572\SD{0.0002} & 0.7021\SD{0.0046} \\
\slciso         & 0.8319\SD{0.0006} & 0.7169\SD{0.0007} & 0.7573\SD{0.0002} & 0.7014\SD{0.0046} \\
\bottomrule
\end{tabular}
\end{table}

The first two rows of Table~\ref{tab:auc} confirm Lemma~\ref{lem:auc}: \base and \platt produce identical AUC.
Across the other score-only calibrators, temperature scaling matches \platt to four decimals of AUC on all four datasets, isotonic changes AUC negligibly, and histogram binning slightly lowers AUC through ties; the appendix reports the full table.
\plattt (score + time, not covered by Lemma~\ref{lem:auc}) changes AUC only marginally ($0.00$ to $+0.23\pp$ across datasets), indicating that time-only recalibration without item identity does not recover the stranded headroom.
The per-item baselines (\rescal, \rescaliso) unlock substantial headroom on AS17 ($+2.5\pp$) and Eedi ($+3.5\pp$) but fail on extremely sparse AS09 ($+0.02\pp$).

\slc improves AUC over \platt on every dataset, with gains correlating with drift intensity: $+3.68\pp$ (AS17), $+3.90\pp$ (Eedi), $+1.90\pp$ (Algebra).
These gains are broad across backbones: \slc improves AUC on all five backbones for AS17, AS09, and Algebra, and on four of five for Eedi; NLL follows a similar pattern, with isolated exceptions on sparse AS09 and one Eedi backbone (per-backbone tables in the appendix).
The most informative comparison is AS09: \rescal extracts negligible headroom ($+0.02\pp$) while \slc recovers $+2.84\pp$, demonstrating that Kalman shrinkage is essential in sparse regimes.
Conversely, \naive (unsmoothed per-item means) \emph{degrades} AUC by $-1.29\pp$ on AS09---unregularized estimation injects more noise than signal.

\begin{table}[t]
\caption{NLL ($\downarrow$) averaged over 5~backbones $\times$ 3~seeds ($\pm$~avg seed std). Best per column in \textbf{bold}.}\label{tab:nll}
\centering
\small
\begin{tabular}{l cccc}
\toprule
Method & Algebra & AS17 & Eedi & AS09 \\
\midrule
\base           & 0.341\SD{0.002} & 0.618\SD{0.001} & 0.579\SD{0.000} & 0.628\SD{0.015} \\
\platt          & 0.338\SD{0.001} & 0.613\SD{0.001} & 0.578\SD{0.000} & 0.592\SD{0.007} \\
\rescal         & 0.334\SD{0.002} & 0.603\SD{0.001} & 0.556\SD{0.000} & 0.627\SD{0.015} \\
\rescaliso      & 0.328\SD{0.001} & 0.600\SD{0.001} & 0.555\SD{0.000} & 0.593\SD{0.007} \\
\naive          & 0.451\SD{0.009} & 0.648\SD{0.003} & 0.579\SD{0.000} & 0.770\SD{0.033} \\
\midrule
\textbf{\slc}   & \textbf{0.326}\SD{0.001} & \textbf{0.593}\SD{0.001} & \textbf{0.552}\SD{0.000} & 0.596\SD{0.012} \\
\slct           & 0.328\SD{0.001} & 0.595\SD{0.001} & 0.553\SD{0.000} & 0.605\SD{0.013} \\
\slciso         & 0.329\SD{0.001} & 0.595\SD{0.001} & 0.553\SD{0.000} & 0.623\SD{0.012} \\
\bottomrule
\end{tabular}
\end{table}

NLL improvements (Table~\ref{tab:nll}) largely mirror AUC: \slc achieves the best NLL on Algebra, AS17, and Eedi.
On AS09, \slc incurs only a small NLL cost over \platt (0.596 vs.\ 0.592) while gaining $+2.84\pp$ AUC.
\naive confirms shrinkage importance: its AS09 NLL (0.770) is catastrophically worse than the uncalibrated backbone (0.628).
\slc's ECE is typically higher than \platt's, a structural consequence of per-item logit correction; since ECE is not a proper scoring rule~\cite{gruber2022better}, we rely on NLL to confirm net probabilistic benefit. Density-stratified analysis (Section~\ref{sec:density}) clarifies this relationship.

\subsection{The Role of Temporal Tracking}
\label{sec:ablations}

The core ablation compares \slc (static $b_i$) against \slct ($b_i + u_i(t)$) under an identical pipeline.
Static estimation produces better AUC in all 20 configurations and equal or better NLL in 19 of 20, with average AUC advantage of $+0.12\pp$ (Algebra), $+0.16\pp$ (AS17), $+0.07\pp$ (Eedi), $+0.46\pp$ (AS09).

\begin{table}[t]
\caption{Temporal drift detectability (Proposition~\ref{prop:detect}). $\delta_{\min}$: Wald detection threshold at $\alpha{=}0.05$; $|\hat\delta_{\text{obs}}|$: median observed adjacent-bin logit change (upper bound on true drift). On every dataset, $|\hat\delta_{\text{obs}}| \ll \delta_{\min}$.}\label{tab:detect}
\centering
\small
\begin{tabular}{lcccccc}
\toprule
Dataset & obs/item & bins & obs/bin & $W_{\max}$ & $\delta_{\min}$ & $|\hat\delta_{\text{obs}}|$ \\
\midrule
AS17    & 48 & 10 & 4.8 & 1.20 & 2.53 & 0.40 \\
Eedi    & 54 & 10 & 5.4 & 1.35 & 2.39 & 0.69 \\
AS09    &  3 &  5 & 0.6 & 0.15 & 7.16 & 1.07 \\
Algebra &  1 &  5 & 0.2 & 0.05 & 12.40 & 0.69 \\
\bottomrule
\end{tabular}
\end{table}

Table~\ref{tab:detect} applies Proposition~\ref{prop:detect} to each dataset: $\delta_{\min}$ exceeds the observed $|\hat\delta_{\text{obs}}|$ by $3.5$--$18{\times}$, consistent with the uniform loss of temporal \slct to static \slc reported above.
Proposition~\ref{prop:detect} also gives a rough viability criterion: for a moderate drift of $0.05$~logit/bin, temporal tracking would require on the order of $10^5$ total obs/item.
Offset-Platt also outperforms raw $\sigmoid(\eta_0 + \hat{b}_i)$ on ECE and NLL (AS09: 7.63\% vs 10.51\%, 0.596 vs 0.621; full breakdown in the appendix).
Empirical-Bayes shrinkage (\slc) dominates unsmoothed naive means by $+2.0$--$4.1\pp$ AUC across datasets.

\subsection{Where Does Kalman Shrinkage Help?}
\label{sec:density}

We stratify items into three density bins by observation count and compute $\Delta$AUC per bin (Table~\ref{tab:density}, Fig.~\ref{fig:density}).

\begin{table}[t]
\caption{Density-stratified $\Delta$AUC (\slc~$-$~\rescaliso) by item observation count. \slc's advantage concentrates on sparse items, where Kalman shrinkage prevents noise injection.}\label{tab:density}
\centering
\small
\begin{tabular}{lccc}
\toprule
Dataset & Bin~0 (sparse) & Bin~1 (medium) & Bin~2 (dense) \\
\midrule
AS17 \scriptsize{(1--23 / 24--74 / $\ge$75)}   & $+4.14\pp$ & $+2.36\pp$ & $+0.24\pp$ \\
Eedi \scriptsize{(1--28 / 29--116 / $\ge$117)}  & $+4.18\pp$ & $+1.28\pp$ & $+0.08\pp$ \\
AS09 \scriptsize{(1--2 / 3--6 / $\ge$7)}        & $+3.67\pp$ & $+2.85\pp$ & $+3.01\pp$ \\
Algebra \scriptsize{($=$1 / $\ge$2)}              & $+0.75\pp$ & --- & $+1.01\pp$ \\
\bottomrule
\end{tabular}
\end{table}

\begin{figure}[t]
  \centering
  \includegraphics[width=0.95\textwidth]{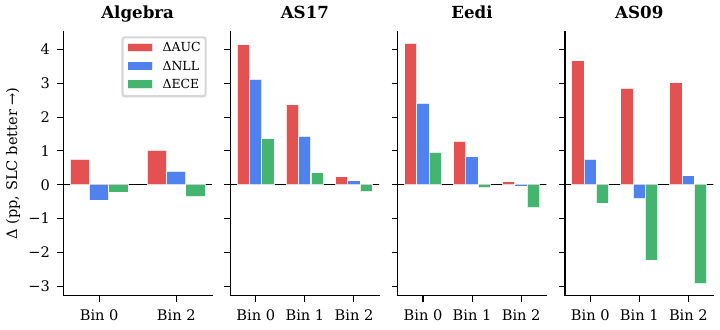}
  \caption{Density-stratified $\Delta$ metrics (\slc~$-$~\rescaliso), averaged over backbones and seeds. On sparse items (Bin~0), \slc achieves simultaneous improvements in AUC, NLL, and ECE---the regime where Kalman shrinkage is most valuable.}
  \label{fig:density}
\end{figure}

On AS17 and Eedi, the advantage decreases monotonically with density ($+4.1\pp$ sparse, $<0.3\pp$ dense); on sparse items \slc also improves NLL and ECE simultaneously.
On AS09, \slc improves across all bins because even the ``dense'' bin ($\ge$7 obs) is globally sparse.

\subsection{Regime Characterization}
\label{sec:regime}

Two analyses characterize when \slc helps and when temporal tracking becomes worthwhile.

\paragraph{Calibration-fraction sweep.}
Varying calibration fraction from 10\% to 100\%, $\Delta$AUC grows monotonically on both datasets (AS17: $+1.24$--$+3.68\pp$; AS09: $+0.16$--$+2.84\pp$).
On AS17, NLL improves in parallel; on AS09, the AUC gain comes with the same small NLL trade-off seen at full calibration fraction (see appendix).

\paragraph{Synthetic regime map.}
To disentangle drift intensity from data density, we run a 1PL-IRT simulation ($K{=}200$, $T{=}20$, 5~seeds) sweeping drift variance $Q$ and observations per item.
Figure~\ref{fig:regime}(a) confirms that $\Delta$AUC of static correction scales with $Q$ and is positive whenever $Q>0$.
Panel~(b) shows that temporal tracking adds $\le 0.7\pp$ even in the most favorable regime (obs$=$300, $Q{=}$0.2), consistent with Proposition~\ref{prop:detect}.
Panel~(c) shows that bias-estimation MSE drops to near zero at $Q{=}0$, confirming that the estimator does not inject spurious corrections when no drift is present.

\begin{figure}[t]
  \centering
  \includegraphics[width=\textwidth]{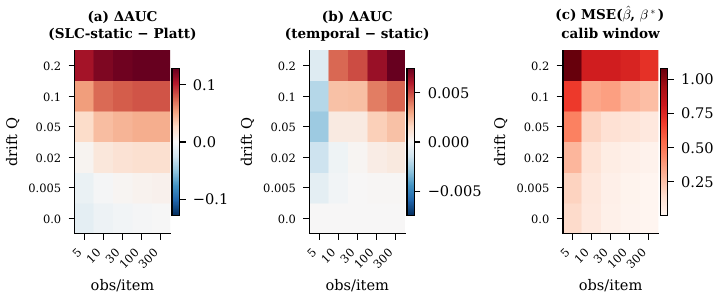}
  \caption{Synthetic regime map (1PL simulation, 5~seeds). \textbf{(a)}~$\Delta$AUC of static per-item correction scales with drift intensity $Q$. \textbf{(b)}~$\Delta$AUC of temporal over static: even at obs$=$300 the gain is $<0.7\pp$. \textbf{(c)}~MSE of bias estimates; the $Q{=}0$ row confirms unbiased recovery.}
  \label{fig:regime}
\end{figure}

\subsection{Cross-Domain Controls}
\label{sec:flight}

We apply the same analysis to US DoT flight-delay data~\cite{bts_ontime} (2018--2019, {$\sim$}12M flights, {$\sim$}2500 routes; SGDClassifier backbone without per-route parameters). Route-aware correction recovers about $+2\pp$ AUC while \platt yields $\Delta\AUC\approx 0$; the less-regularized \rescal is slightly stronger in this dense regime (appendix).
As a negative control, MovieLens-1M~\cite{harper2015movielens} with MF~\cite{koren2009mf}/NCF~\cite{he2017ncf} backbones yields $\Delta\AUC\approx 0$. Together, the controls confirm that stranded headroom is backbone-relative: it appears when the backbone leaves residual entity-level bias and weakens once those effects are already modeled.

\section{Discussion and Limitations}
\label{sec:limitations}

As a batch method, \slc should be refreshed on roughly the same cadence as global calibration.
In extremely sparse settings (AS09: median 3~obs/item), Kalman shrinkage still improves AUC but can incur a small NLL cost vs.\ \platt{} ($+0.004$).
The additive per-item form (Proposition~\ref{prop:additive}) is MSE-optimal among item-only corrections but cannot capture student\,{$\times$}\,item interactions; consistent gains across five backbones suggest the item-marginal component dominates.
The effect is backbone-relative: \slc adds less in dense regimes (flight-delay) or when item effects are already modeled (MovieLens).
Temporal tracking provides no benefit at current densities; Proposition~\ref{prop:detect} predicts a viability threshold on the order of~$10^5$ obs/item.
Student- and skill-level random effects and fairness audits are left to future work.

\section{Conclusion}
\label{sec:conclusion}

Global score-only calibration is structurally AUC-invariant; per-item shrinkage correction recovers stranded discrimination that no global calibrator can access.
\slc is lightweight ($O(N{+}KT)$), backbone-agnostic, and concentrates its advantage on sparse items where shrinkage prevents catastrophic noise injection.
Temporal tracking is information-limited at current KT densities; Proposition~\ref{prop:detect} provides a quantitative viability criterion for future, denser deployments.

\begin{credits}
\subsubsection{\ackname}
This work was supported by JST CREST Grant Number JPMJCR22D1, Japan.

\subsubsection{\discintname}
The authors have no competing interests to declare that are relevant to the content of this article.
\end{credits}

\bibliographystyle{splncs04}
\bibliography{references}

\appendix
\section*{Appendix}
\input{appendix_body}

\end{document}

%% file: appendix_body.tex
\section{Per-Backbone Results}
\label{app:perbackbone}

Tables~\ref{tab:auc_perbb} and~\ref{tab:nll_perbb} report AUC and NLL for each backbone individually (mean $\pm$~std across 3~seeds).
Per-backbone results show that the AUC gain is broad across architectures, with AKT on Eedi ($-0.19\pp$) as the only exception; there, AKT's item-attention mechanism already captures most per-item variation, leaving little headroom for post-hoc correction.
SAKT consistently shows the largest gains ($+2.98$--$6.17\pp$), consistent with its weaker per-item modeling capacity.

\begin{table}[h]
\caption{Per-backbone AUC ($\uparrow$), mean $\pm$~std across 3~seeds. Best per column in \textbf{bold}.}\label{tab:auc_perbb}
\centering
\scriptsize
\setlength{\tabcolsep}{1.5pt}
\begin{tabular}{ll ccccc}
\toprule
Dataset & Method & AKT & DKT & SAKT & DKVMN & LPKT \\
\midrule
\multirow{9}{*}{AS17}
 & \base       & 0.7189\SD{0.003} & 0.6991\SD{0.000} & 0.6193\SD{0.001} & 0.6841\SD{0.001} & 0.6855\SD{0.000} \\
 & \platt      & 0.7189\SD{0.003} & 0.6991\SD{0.000} & 0.6193\SD{0.001} & 0.6841\SD{0.001} & 0.6855\SD{0.000} \\
 & \plattt     & 0.7198\SD{0.003} & 0.6991\SD{0.000} & 0.6193\SD{0.001} & 0.6841\SD{0.001} & 0.6855\SD{0.000} \\
 & \rescal     & 0.7275\SD{0.002} & 0.7222\SD{0.001} & 0.6613\SD{0.001} & 0.7096\SD{0.000} & 0.7117\SD{0.000} \\
 & \rescaliso  & 0.7274\SD{0.002} & 0.7221\SD{0.001} & 0.6613\SD{0.001} & 0.7094\SD{0.000} & 0.7118\SD{0.000} \\
 & \naive      & 0.7060\SD{0.002} & 0.7079\SD{0.001} & 0.6584\SD{0.001} & 0.6984\SD{0.000} & 0.7022\SD{0.000} \\
 & \textbf{\slc} & \textbf{0.7321}\SD{0.002} & \textbf{0.7325}\SD{0.001} & \textbf{0.6811}\SD{0.001} & \textbf{0.7214}\SD{0.000} & \textbf{0.7237}\SD{0.000} \\
 & \slct       & 0.7303\SD{0.002} & 0.7312\SD{0.001} & 0.6791\SD{0.000} & 0.7201\SD{0.000} & 0.7222\SD{0.000} \\
 & \slciso     & 0.7314\SD{0.002} & 0.7314\SD{0.001} & 0.6790\SD{0.000} & 0.7204\SD{0.000} & 0.7225\SD{0.000} \\
\midrule
\multirow{9}{*}{Eedi}
 & \base       & 0.7733\SD{0.000} & 0.7207\SD{0.000} & 0.7196\SD{0.000} & 0.7219\SD{0.000} & 0.6591\SD{0.000} \\
 & \platt      & 0.7733\SD{0.000} & 0.7207\SD{0.000} & 0.7196\SD{0.000} & 0.7219\SD{0.000} & 0.6591\SD{0.000} \\
 & \plattt     & 0.7733\SD{0.000} & 0.7207\SD{0.000} & 0.7196\SD{0.000} & 0.7219\SD{0.000} & 0.6594\SD{0.000} \\
 & \rescal     & 0.7734\SD{0.000} & 0.7589\SD{0.000} & 0.7568\SD{0.000} & 0.7592\SD{0.000} & 0.7188\SD{0.000} \\
 & \rescaliso  & 0.7734\SD{0.000} & 0.7589\SD{0.000} & 0.7568\SD{0.000} & 0.7592\SD{0.000} & 0.7188\SD{0.000} \\
 & \naive      & 0.7509\SD{0.000} & 0.7433\SD{0.000} & 0.7433\SD{0.000} & 0.7447\SD{0.000} & 0.7051\SD{0.000} \\
 & \textbf{\slc} & 0.7714\SD{0.000} & \textbf{0.7643}\SD{0.000} & \textbf{0.7628}\SD{0.000} & \textbf{0.7650}\SD{0.000} & \textbf{0.7259}\SD{0.000} \\
 & \slct       & 0.7704\SD{0.000} & 0.7637\SD{0.000} & 0.7622\SD{0.000} & 0.7644\SD{0.000} & 0.7255\SD{0.000} \\
 & \slciso     & 0.7709\SD{0.000} & 0.7634\SD{0.000} & 0.7622\SD{0.000} & 0.7643\SD{0.000} & 0.7254\SD{0.000} \\
\midrule
\multirow{9}{*}{AS09}
 & \base       & 0.6925\SD{0.009} & 0.6987\SD{0.005} & 0.6084\SD{0.012} & 0.6817\SD{0.007} & 0.7099\SD{0.001} \\
 & \platt      & 0.6925\SD{0.009} & 0.6987\SD{0.005} & 0.6084\SD{0.012} & 0.6817\SD{0.007} & 0.7099\SD{0.001} \\
 & \plattt     & 0.6918\SD{0.008} & 0.6994\SD{0.004} & 0.6113\SD{0.012} & 0.6820\SD{0.007} & 0.7087\SD{0.002} \\
 & \rescal     & 0.6928\SD{0.009} & 0.6989\SD{0.005} & 0.6085\SD{0.012} & 0.6819\SD{0.007} & 0.7102\SD{0.001} \\
 & \rescaliso  & 0.6922\SD{0.008} & 0.6986\SD{0.005} & 0.6085\SD{0.013} & 0.6819\SD{0.007} & 0.7101\SD{0.001} \\
 & \naive      & 0.6712\SD{0.006} & 0.6706\SD{0.002} & 0.6316\SD{0.003} & 0.6686\SD{0.002} & 0.6847\SD{0.001} \\
 & \textbf{\slc} & \textbf{0.7131}\SD{0.011} & \textbf{0.7198}\SD{0.004} & \textbf{0.6570}\SD{0.003} & \textbf{0.7078}\SD{0.005} & \textbf{0.7355}\SD{0.002} \\
 & \slct       & 0.7089\SD{0.010} & 0.7147\SD{0.004} & 0.6532\SD{0.003} & 0.7035\SD{0.005} & 0.7300\SD{0.002} \\
 & \slciso     & 0.7108\SD{0.010} & 0.7123\SD{0.003} & 0.6512\SD{0.004} & 0.7033\SD{0.005} & 0.7295\SD{0.001} \\
\midrule
\multirow{9}{*}{Algebra}
 & \base       & 0.8168\SD{0.000} & 0.8254\SD{0.000} & 0.7918\SD{0.001} & 0.8179\SD{0.000} & 0.8158\SD{0.001} \\
 & \platt      & 0.8168\SD{0.000} & 0.8254\SD{0.000} & 0.7918\SD{0.001} & 0.8179\SD{0.000} & 0.8158\SD{0.001} \\
 & \plattt     & 0.8169\SD{0.000} & 0.8254\SD{0.000} & 0.7918\SD{0.001} & 0.8179\SD{0.000} & 0.8158\SD{0.001} \\
 & \rescal     & 0.8204\SD{0.001} & 0.8344\SD{0.000} & 0.8094\SD{0.001} & 0.8290\SD{0.000} & 0.8281\SD{0.001} \\
 & \rescaliso  & 0.8204\SD{0.001} & 0.8343\SD{0.000} & 0.8095\SD{0.001} & 0.8290\SD{0.000} & 0.8281\SD{0.001} \\
 & \naive      & 0.7947\SD{0.001} & 0.8206\SD{0.000} & 0.7998\SD{0.002} & 0.8167\SD{0.001} & 0.8168\SD{0.000} \\
 & \textbf{\slc} & \textbf{0.8245}\SD{0.001} & \textbf{0.8417}\SD{0.000} & \textbf{0.8216}\SD{0.001} & \textbf{0.8377}\SD{0.000} & \textbf{0.8372}\SD{0.000} \\
 & \slct       & 0.8225\SD{0.001} & 0.8408\SD{0.000} & 0.8206\SD{0.001} & 0.8367\SD{0.000} & 0.8362\SD{0.000} \\
 & \slciso     & 0.8242\SD{0.001} & 0.8409\SD{0.000} & 0.8209\SD{0.001} & 0.8369\SD{0.000} & 0.8365\SD{0.000} \\
\bottomrule
\end{tabular}
\end{table}

\begin{table}[h]
\caption{Per-backbone NLL ($\downarrow$), mean $\pm$~std across 3~seeds. Best per column in \textbf{bold}.}\label{tab:nll_perbb}
\centering
\scriptsize
\setlength{\tabcolsep}{3pt}
\begin{tabular}{ll ccccc}
\toprule
Dataset & Method & AKT & DKT & SAKT & DKVMN & LPKT \\
\midrule
\multirow{9}{*}{AS17}
 & \base       & 0.603\SD{0.003} & 0.607\SD{0.001} & 0.651\SD{0.001} & 0.617\SD{0.001} & 0.615\SD{0.000} \\
 & \platt      & 0.593\SD{0.002} & 0.605\SD{0.000} & 0.644\SD{0.000} & 0.614\SD{0.000} & 0.612\SD{0.000} \\
 & \plattt     & 0.592\SD{0.002} & 0.604\SD{0.000} & 0.643\SD{0.000} & 0.613\SD{0.000} & 0.611\SD{0.000} \\
 & \rescal     & 0.595\SD{0.003} & 0.591\SD{0.001} & 0.630\SD{0.001} & 0.601\SD{0.001} & 0.599\SD{0.000} \\
 & \rescaliso  & 0.587\SD{0.002} & 0.590\SD{0.001} & 0.627\SD{0.001} & 0.599\SD{0.000} & 0.598\SD{0.000} \\
 & \naive      & 0.662\SD{0.005} & 0.635\SD{0.004} & 0.666\SD{0.001} & 0.642\SD{0.002} & 0.635\SD{0.001} \\
 & \textbf{\slc} & \textbf{0.585}\SD{0.002} & \textbf{0.583}\SD{0.001} & \textbf{0.616}\SD{0.000} & \textbf{0.592}\SD{0.000} & \textbf{0.590}\SD{0.000} \\
 & \slct       & 0.587\SD{0.002} & 0.585\SD{0.001} & 0.618\SD{0.000} & 0.594\SD{0.000} & 0.591\SD{0.000} \\
 & \slciso     & 0.585\SD{0.002} & 0.585\SD{0.001} & 0.618\SD{0.000} & 0.594\SD{0.000} & 0.592\SD{0.000} \\
\midrule
\multirow{9}{*}{Eedi}
 & \base       & 0.540\SD{0.000} & 0.580\SD{0.000} & 0.582\SD{0.001} & 0.579\SD{0.000} & 0.612\SD{0.000} \\
 & \platt      & 0.539\SD{0.000} & 0.580\SD{0.000} & 0.581\SD{0.000} & 0.579\SD{0.000} & 0.612\SD{0.000} \\
 & \plattt     & 0.539\SD{0.000} & 0.580\SD{0.000} & 0.581\SD{0.000} & 0.579\SD{0.000} & 0.611\SD{0.000} \\
 & \rescal     & 0.540\SD{0.000} & 0.552\SD{0.000} & 0.554\SD{0.001} & 0.552\SD{0.000} & 0.581\SD{0.000} \\
 & \rescaliso  & 0.539\SD{0.000} & 0.552\SD{0.000} & 0.553\SD{0.000} & 0.551\SD{0.000} & 0.580\SD{0.000} \\
 & \naive      & 0.574\SD{0.001} & 0.573\SD{0.000} & 0.576\SD{0.001} & 0.574\SD{0.000} & 0.599\SD{0.000} \\
 & \textbf{\slc} & 0.541\SD{0.000} & \textbf{0.548}\SD{0.000} & \textbf{0.549}\SD{0.000} & \textbf{0.548}\SD{0.000} & \textbf{0.576}\SD{0.000} \\
 & \slct       & 0.542\SD{0.000} & 0.549\SD{0.000} & 0.550\SD{0.000} & 0.548\SD{0.000} & 0.576\SD{0.000} \\
 & \slciso     & 0.542\SD{0.000} & 0.548\SD{0.000} & 0.549\SD{0.000} & 0.547\SD{0.000} & 0.576\SD{0.000} \\
\midrule
\multirow{9}{*}{AS09}
 & \base       & 0.630\SD{0.035} & 0.590\SD{0.004} & 0.710\SD{0.012} & 0.629\SD{0.020} & 0.579\SD{0.003} \\
 & \platt      & 0.593\SD{0.017} & 0.577\SD{0.004} & 0.625\SD{0.004} & 0.598\SD{0.009} & 0.569\SD{0.002} \\
 & \plattt     & 0.593\SD{0.016} & 0.575\SD{0.004} & 0.621\SD{0.002} & 0.596\SD{0.008} & 0.567\SD{0.002} \\
 & \rescal     & 0.629\SD{0.035} & 0.589\SD{0.004} & 0.708\SD{0.012} & 0.629\SD{0.020} & 0.578\SD{0.003} \\
 & \rescaliso  & 0.594\SD{0.017} & 0.577\SD{0.004} & 0.625\SD{0.004} & 0.598\SD{0.007} & 0.571\SD{0.002} \\
 & \naive      & 0.781\SD{0.069} & 0.702\SD{0.010} & 0.917\SD{0.055} & 0.758\SD{0.028} & 0.694\SD{0.004} \\
 & \textbf{\slc} & 0.599\SD{0.025} & \textbf{0.575}\SD{0.006} & 0.636\SD{0.009} & 0.606\SD{0.015} & \textbf{0.563}\SD{0.003} \\
 & \slct       & 0.608\SD{0.027} & 0.582\SD{0.006} & 0.649\SD{0.011} & 0.617\SD{0.017} & 0.570\SD{0.003} \\
 & \slciso     & 0.622\SD{0.029} & 0.610\SD{0.008} & 0.653\SD{0.007} & 0.640\SD{0.011} & 0.590\SD{0.003} \\
\midrule
\multirow{9}{*}{Algebra}
 & \base       & 0.348\SD{0.006} & 0.328\SD{0.001} & 0.362\SD{0.003} & 0.335\SD{0.000} & 0.335\SD{0.000} \\
 & \platt      & 0.340\SD{0.003} & 0.327\SD{0.001} & 0.357\SD{0.001} & 0.334\SD{0.001} & 0.334\SD{0.000} \\
 & \plattt     & 0.340\SD{0.002} & 0.327\SD{0.001} & 0.357\SD{0.001} & 0.334\SD{0.001} & 0.334\SD{0.000} \\
 & \rescal     & 0.345\SD{0.006} & 0.322\SD{0.001} & 0.350\SD{0.003} & 0.327\SD{0.000} & 0.326\SD{0.000} \\
 & \rescaliso  & 0.331\SD{0.001} & 0.320\SD{0.001} & 0.339\SD{0.001} & 0.325\SD{0.001} & 0.325\SD{0.000} \\
 & \naive      & 0.557\SD{0.011} & 0.410\SD{0.002} & 0.462\SD{0.025} & 0.418\SD{0.002} & 0.410\SD{0.003} \\
 & \textbf{\slc} & \textbf{0.336}\SD{0.003} & \textbf{0.316}\SD{0.001} & \textbf{0.339}\SD{0.002} & \textbf{0.320}\SD{0.001} & \textbf{0.319}\SD{0.000} \\
 & \slct       & 0.340\SD{0.003} & 0.317\SD{0.001} & 0.342\SD{0.002} & 0.322\SD{0.001} & 0.321\SD{0.000} \\
 & \slciso     & 0.335\SD{0.002} & 0.321\SD{0.001} & 0.340\SD{0.001} & 0.325\SD{0.001} & 0.324\SD{0.000} \\
\bottomrule
\end{tabular}
\end{table}

\section{Classic Score-Only Calibrators}
\label{app:classic_calibrators}

Table~\ref{tab:classic_calibrators} makes the score-only comparison visible: temperature scaling matches \platt in AUC to four decimals on all four datasets, isotonic regression changes AUC only negligibly, and histogram binning can lower AUC through ties.
These results are averaged over the same five backbones as the main-text tables.

\begin{table}[h]
\caption{Classic score-only calibrators, averaged over 5 backbones at full calibration fraction. Temperature scaling matches \platt in AUC; isotonic regression changes AUC only marginally; histogram binning can reduce AUC through ties.}\label{tab:classic_calibrators}
\centering
\small
\setlength{\tabcolsep}{4pt}
\begin{tabular}{l cc cc cc cc}
\toprule
 & \multicolumn{2}{c}{Algebra} & \multicolumn{2}{c}{AS17} & \multicolumn{2}{c}{Eedi} & \multicolumn{2}{c}{AS09} \\
\cmidrule(lr){2-3} \cmidrule(lr){4-5} \cmidrule(lr){6-7} \cmidrule(lr){8-9}
Method & AUC & NLL & AUC & NLL & AUC & NLL & AUC & NLL \\
\midrule
\base   & 0.8135 & 0.341 & 0.6814 & 0.618 & 0.7189 & 0.579 & 0.6782 & 0.628 \\
\platt  & 0.8135 & 0.338 & 0.6814 & 0.613 & 0.7189 & 0.578 & 0.6782 & 0.592 \\
TS      & 0.8135 & 0.339 & 0.6814 & 0.614 & 0.7189 & 0.578 & 0.6782 & 0.613 \\
\iso    & 0.8135 & 0.335 & 0.6813 & 0.614 & 0.7189 & 0.578 & 0.6780 & 0.593 \\
\hist   & 0.8006 & 0.340 & 0.6793 & 0.614 & 0.7162 & 0.579 & 0.6766 & 0.592 \\
\bottomrule
\end{tabular}
\end{table}

\section{Sensitivity to Prior Variance $\sigma_b^2$}
\label{app:sigma_b}

The static \slc estimator (Eq.~6 in the main text) uses a fixed prior variance $\sigma_b^2 = 1.0$ for the per-item bias.
This appendix reports a controlled sensitivity sweep over $\sigma_b^2 \in \{0.01, 0.1, 0.5, 1.0, 5.0, 10.0, 100.0\}$ and explains why the result is backbone-independent.

We load the AS17 calibration and test splits (188\,554 and 188\,565 tokens, 3\,162 items; median 48 observations per item).
A constant backbone $p_0 = \bar{y}_{\text{calib}} \approx 0.55$ replaces the trained KT model.
For each $\sigma_b^2$ value we:
(i)~compute the static per-item bias $\hat{b}_i$ on the calibration set via Eq.~6,
(ii)~fit the two-parameter Offset-Platt link $(a^*, b_0^*)$ by IRLS on the same calibration set, and
(iii)~evaluate AUC and NLL on the held-out test set.

A constant backbone is sufficient for this sweep because the sensitivity of $\hat{b}_i$ to $\sigma_b^2$ is governed entirely by the shrinkage fraction
$\lambda_i = W_i / (1/\sigma_b^2 + W_i)$,
where $W_i = \sum_n p_n(1{-}p_n)$ is the total Fisher-information weight of item~$i$.
For a constant backbone, $W_i = n_i \cdot p_0(1{-}p_0) \approx 0.2475\,n_i$;
for any trained backbone with predictions spread over $[0,1]$, the per-token weight $p_n(1{-}p_n) \in [0, 0.25]$ averages to $\approx 0.20$--$0.24$, so $W_i$ differs by $<15\%$.
Because $\lambda_i$ is monotone in $W_i$ and saturates rapidly for $W_i \gg 1/\sigma_b^2$, the sensitivity pattern is backbone-invariant.
The constant backbone yields a slightly \emph{higher} $W_i$ (since $p(1{-}p)$ is maximized at $p = 0.5$), making this a conservative test: a trained backbone would show even less sensitivity.

Table~\ref{tab:sigma_b_sweep} reports the sweep.
Over the two-order-of-magnitude range $\sigma_b^2 \in [0.1, 10]$, AUC varies by $<0.4\pp$ and NLL by $<0.004$.
Only the extreme $\sigma_b^2 = 0.01$ (prior precision $= 100$, equivalent to heavy shrinkage toward zero) shows a visible AUC loss of ${\sim}0.9\pp$, consistent with excessive regularization suppressing genuine item biases.

\begin{table}[h]
\caption{Sensitivity of \slc to prior variance $\sigma_b^2$ on AS17 (constant backbone). The default $\sigma_b^2 = 1.0$ is shaded.}\label{tab:sigma_b_sweep}
\centering
\small
\begin{tabular}{r cc cc}
\toprule
$\sigma_b^2$ & AUC ($\uparrow$) & $\Delta$AUC (pp) & NLL ($\downarrow$) & $\Delta$NLL \\
\midrule
0.01  & 0.6445 & $+$14.45 & 0.6522 & $-$0.0147 \\
0.10  & 0.6504 & $+$15.04 & 0.6348 & $-$0.0321 \\
0.50  & 0.6532 & $+$15.32 & 0.6315 & $-$0.0354 \\
\rowcolor{black!8} 1.00  & 0.6537 & $+$15.37 & 0.6315 & $-$0.0354 \\
5.00  & 0.6535 & $+$15.35 & 0.6329 & $-$0.0340 \\
10.00 & 0.6534 & $+$15.34 & 0.6338 & $-$0.0331 \\
100.00& 0.6533 & $+$15.33 & 0.6354 & $-$0.0315 \\
\midrule
\multicolumn{5}{l}{\footnotesize Span over $[0.1, 10]$: AUC $= 0.33\pp$, NLL $= 0.003$.} \\
\bottomrule
\end{tabular}
\end{table}

\section{Connection Between the Static Correction Block and Ridge Logistic Regression}
\label{app:ridge_equiv}

We show that the iterative version of the static per-item correction block is exactly the corresponding IRLS procedure for $\ell_2$-penalized logistic regression with per-item intercepts and the backbone logit as a fixed offset.
This clarifies the objective behind Eq.~6 in the main text.
The deployed \slc in Algorithm~1 uses the associated single-pass blockwise estimate, followed by a separate offset-Platt fit.

\subsection{Setup}

Consider $N$ observations with binary labels $y_n \in \{0,1\}$, frozen backbone logits $\eta_{0,n}$, and item assignments $i_n \in \{1,\ldots,K\}$.
The per-item correction model is $p_n = \sigma(\eta_{0,n} + b_{i_n})$, where $\sigma$ denotes the sigmoid function.
The $\ell_2$-penalized negative log-likelihood (``ridge logistic regression'') is:
\begin{equation}\label{eq:ridge_obj}
  \mathcal{L}(\mathbf{b}) = -\sum_{n=1}^{N} \bigl[ y_n \log p_n + (1{-}y_n)\log(1{-}p_n) \bigr] + \frac{\lambda}{2}\sum_{i=1}^{K} b_i^2, \qquad \lambda = 1/\sigma_b^2.
\end{equation}

\subsection{Hessian is Diagonal}

The gradient of $\mathcal{L}$ with respect to $b_i$ is:
\[
  \frac{\partial \mathcal{L}}{\partial b_i} = -\!\sum_{n:\, i_n=i}\!(y_n - p_n) + \lambda\, b_i \;=\; -g_i + \lambda\, b_i,
\]
where $g_i \triangleq \sum_{n:\,i_n=i}(y_n - p_n)$ is the per-item score residual.

The Hessian entries are:
\[
  \frac{\partial^2 \mathcal{L}}{\partial b_i \,\partial b_j}
  = \begin{cases}
      W_i + \lambda & \text{if } i = j, \\
      0             & \text{if } i \neq j,
    \end{cases}
\]
where $W_i = \sum_{n:\,i_n=i} p_n(1{-}p_n)$ is the Fisher information weight.
\textbf{The off-diagonal is zero} because each observation $n$ involves exactly one item $i_n$: the derivative $\partial p_n / \partial b_j = 0$ whenever $j \neq i_n$.
Hence $H = \mathrm{diag}(W_1{+}\lambda, \ldots, W_K{+}\lambda)$.

\subsection{Newton Step Decomposes into $K$ Independent Scalar Updates}

The Newton--Raphson update $\mathbf{b}^{(\ell+1)} = \mathbf{b}^{(\ell)} - H^{-1}\nabla\mathcal{L}$ reduces to $K$ independent updates:
\begin{equation}\label{eq:ridge_update}
  b_i^{(\ell+1)}
  = b_i^{(\ell)} + \frac{g_i - \lambda\, b_i^{(\ell)}}{W_i + \lambda}
  = \frac{W_i \cdot b_i^{(\ell)} + g_i}{W_i + 1/\sigma_b^2}.
\end{equation}

\subsection{Comparison with the Iterated Static Correction Block}

The iterated static correction block computes:
\begin{enumerate}[nosep]
\item $p_n = \sigma(\eta_{0,n} + b_{i_n}^{(\ell)})$ for all $n$,
\item $W_i = \sum_{n:\,i_n=i} p_n(1{-}p_n)$, \quad $g_i = \sum_{n:\,i_n=i}(y_n - p_n)$,
\item $b_i^{(\ell+1)} = \frac{W_i \cdot b_i^{(\ell)} + g_i}{1/\sigma_b^2 + W_i}$.
\end{enumerate}

\textbf{This is identical to Eq.~\eqref{eq:ridge_update}.}
Given the same initialization $b_i^{(0)} = 0$, the iterated blockwise solver and ridge-logistic IRLS produce the same updates at every step and converge to the same fixed point.

\begin{quote}
\textbf{Proposition (Blockwise ridge connection).}
Under the conditions that (i) the backbone logit $\eta_{0,n}$ is a fixed offset, (ii) $\ell_2$ penalty $\lambda = 1/\sigma_b^2$ is applied only to the item intercepts $b_i$, and (iii) initialization is $b_i^{(0)} = 0$, the repeated blockwise updates of the static correction block are identical to the IRLS iterates of the corresponding penalized logistic model.
The converged block estimator is the posterior mean of a Gaussian random-intercept model under the Laplace approximation.
\end{quote}

The iterated blockwise solver and ridge logistic converge to the same static correction solution, but differ in implementation.
Standard ridge logistic regression constructs a design matrix of dimension $N \times K$ and solves a $K$-dimensional system at each Newton step, yielding $O(NK)$ cost per iteration.
\slc exploits the diagonal Hessian structure by computing per-item \texttt{bincount} aggregates ($O(N)$) followed by $K$ scalar divisions ($O(K)$), for a total cost of $O(N + K)$ per iteration.
This decomposition also enables seamless extension to temporal smoothing via the Rauch--Tung--Striebel smoother at cost $O(K \cdot T)$.

\subsection{Empirical Verification}

To confirm the theoretical equivalence, we run a fully-converged ridge logistic regression baseline (50~IRLS iterations, same $\lambda{=}1/\sigma_b^2$, followed by the same offset-Platt link) alongside the single-pass \slc on two datasets and two backbones (3~seeds each).
Table~\ref{tab:ridge_equiv} reports the difference: the maximum discrepancy is $0.07\pp$ in AUC and $0.001$ in NLL, both well within seed-level variance.
The sign of $\Delta$AUC is inconsistent across configurations, confirming that the residual gap is noise rather than a systematic bias from the single-pass approximation.

\begin{table}[h]
\caption{Ridge logistic regression vs.\ \slc (single-pass).
$\Delta = \text{Ridge} - \text{\slc}$, averaged over 3~seeds ($\pm$~std).
All differences are within seed-level noise.}\label{tab:ridge_equiv}
\centering
\small
\begin{tabular}{ll cc cc}
\toprule
Dataset & Backbone & Ridge AUC & \slc AUC & $\Delta$AUC & $\Delta$NLL \\
\midrule
AS17 & AKT & 0.7327\SD{0.0021} & 0.7321\SD{0.0025} & $+$0.07\pp & $-$0.001 \\
AS17 & DKT & 0.7329\SD{0.0010} & 0.7325\SD{0.0011} & $+$0.04\pp & $<$0.001 \\
AS09 & AKT & 0.7132\SD{0.0130} & 0.7131\SD{0.0130} & $+$0.01\pp & $-$0.001 \\
AS09 & DKT & 0.7195\SD{0.0044} & 0.7198\SD{0.0045} & $-$0.03\pp & $+$0.001 \\
\bottomrule
\end{tabular}
\end{table}

If the backbone scale $a$ is jointly optimized with $b_i$ (as in standard sklearn \texttt{LogisticRegression} with per-item one-hot features), the Hessian acquires off-diagonal blocks $\partial^2\mathcal{L}/\partial a\,\partial b_i \neq 0$, and the global Newton step no longer decomposes.
\slc therefore separates bias estimation ($b_i$ only) from link estimation ($(a, b_0)$ via offset-Platt), preserving the additive-offset parameterization of Corollary~1 in the main text without claiming a joint Newton step for the full pipeline.

\section{Link Ablation: Offset-Platt vs.\ Raw Sigmoid}
\label{app:link_ablation}

The main text (Section~4.3) reports that offset-Platt outperforms raw $\sigma(\eta_0 + \hat{b}_i)$ on ECE and NLL.
Table~\ref{tab:link_ablation} provides the full breakdown across datasets and link types.

\begin{table}[h]
\caption{Link ablation on AS09 and AS17 (5~backbones $\times$ 3~seeds, averaged). \textbf{Raw}: $\sigma(\eta_0 + \hat{b}_i)$; \textbf{+OP}: offset-Platt $\sigma(a\eta_0 + b_0 + \hat{b}_i)$. Static variants use $b_i$ only; temporal variants add $u_i(t)$.}\label{tab:link_ablation}
\centering
\small
\begin{tabular}{ll ccc}
\toprule
Dataset & Method & AUC ($\uparrow$) & NLL ($\downarrow$) & ECE ($\downarrow$) \\
\midrule
\multirow{6}{*}{AS09}
 & \platt (no per-item) & 0.6782 & 0.5924 & 0.0498 \\
 & Raw temporal        & 0.7016 & 0.6275 & 0.1071 \\
 & Raw temporal + OP   & 0.7021 & 0.6052 & 0.0830 \\
 & Raw static          & 0.7060 & 0.6206 & 0.1051 \\
 & \textbf{Static + OP (\slc)} & \textbf{0.7066} & \textbf{0.5956} & \textbf{0.0763} \\
\midrule
\multirow{6}{*}{AS17}
 & \platt (no per-item) & 0.6814 & 0.6134 & 0.0116 \\
 & Raw temporal        & 0.7170 & 0.5971 & 0.0317 \\
 & Raw temporal + OP   & 0.7166 & 0.5951 & 0.0214 \\
 & Raw static          & 0.7182 & 0.5955 & 0.0279 \\
 & \textbf{Static + OP (\slc)} & \textbf{0.7182} & \textbf{0.5931} & \textbf{0.0170} \\
\bottomrule
\end{tabular}
\end{table}

Adding offset-Platt consistently improves ECE and NLL without sacrificing AUC.
On AS09: offset-Platt reduces ECE from 10.51\% to 7.63\% ($-2.88\pp$) and NLL from 0.621 to 0.596 ($-0.025$).
On AS17: ECE from 2.79\% to 1.70\% ($-1.09\pp$) and NLL from 0.596 to 0.593 ($-0.002$).
The offset-Platt link correctly treats $\hat{b}_i$ as a random effect that should not be rescaled by the global parameter $a$, matching the GLMM parameterization (Corollary~1 in the main text).

\section{Calibration-Fraction Sweep}
\label{app:calib_frac}

The main text (Section~4.5) reports that $\Delta$AUC scales monotonically with calibration fraction.
Table~\ref{tab:calib_frac} provides the detailed results on AS17 and AS09, averaged over 5~backbones $\times$ 3~seeds.
All numbers in Table~\ref{tab:calib_frac} use the final single-pass \slc pipeline.

\begin{table}[h]
\caption{Calibration-fraction sweep on AS17 and AS09 (5~backbones $\times$ 3~seeds, averaged). \slc's $\Delta$AUC over \platt scales monotonically with calibration data, confirming signal-driven improvement.}\label{tab:calib_frac}
\centering
\small
\begin{tabular}{l cccc cccc}
\toprule
 & \multicolumn{4}{c}{AS17} & \multicolumn{4}{c}{AS09} \\
\cmidrule(lr){2-5} \cmidrule(lr){6-9}
Calib frac & \platt & \slc & $\Delta$AUC & $\Delta$NLL & \platt & \slc & $\Delta$AUC & $\Delta$NLL \\
\midrule
10\% & 0.6814 & 0.6938 & $+1.24$ & $-0.003$ & 0.6782 & 0.6798 & $+0.16$ & $+0.008$ \\
20\% & 0.6814 & 0.7001 & $+1.87$ & $-0.007$ & 0.6782 & 0.6830 & $+0.48$ & $+0.010$ \\
50\% & 0.6814 & 0.7090 & $+2.76$ & $-0.013$ & 0.6782 & 0.6866 & $+0.84$ & $+0.012$ \\
100\% & 0.6814 & 0.7182 & $+3.68$ & $-0.020$ & 0.6782 & 0.7066 & $+2.84$ & $+0.003$ \\
\bottomrule
\end{tabular}
\end{table}

On AS17, both $\Delta$AUC and $\Delta$NLL improve monotonically as more calibration data become available, confirming that the gain is signal-driven rather than an artifact of a particular split.
On AS09, $\Delta$AUC is also monotone, but NLL remains slightly above \platt throughout; the gap shrinks from $+0.012$ at 50\% calibration data to $+0.003$ at 100\%.
This matches the main-text conclusion that in extremely sparse regimes, \slc recovers substantial ranking headroom while retaining a small proper-score trade-off against \platt.

\section{Flight-Delay and MovieLens Experiments}
\label{app:generality}

\subsection{Flight-Delay (Positive Control)}

The main text (Section~4.6) uses flight-delay as a positive control for the claim that the same phenomenon can arise beyond education when the deployed backbone leaves route-level bias.
Table~\ref{tab:flight} provides the full results across two backbone variants and all methods.

We use US Department of Transportation on-time performance data from 2018--2019 ({$\sim$}12M flights, {$\sim$}2500 routes).
We evaluate two backbone variants: \textsf{bbA} (SGDClassifier with carrier/origin/dest features, no per-route parameters) and \textsf{bbB} (the same model family with different regularization).
The temporal split uses 2018 for training, early 2019 for calibration, and late 2019 for test, with three random seeds for backbone training.

\begin{table}[h]
\caption{Flight-delay experiment (2~backbones $\times$ 3~seeds, mean $\pm$~std). Route-aware correction recovers AUC headroom that score-only calibration cannot access; in this dense regime, the less-regularized \rescal baseline attains the highest AUC.}\label{tab:flight}
\centering
\small
\setlength{\tabcolsep}{4pt}
\begin{tabular}{l cc cc}
\toprule
 & \multicolumn{2}{c}{bbA} & \multicolumn{2}{c}{bbB} \\
\cmidrule(lr){2-3} \cmidrule(lr){4-5}
Method & AUC & NLL & AUC & NLL \\
\midrule
\base            & 0.5813\SD{0.012} & 0.5511 & 0.5823\SD{0.012} & 0.5763 \\
\platt           & 0.5813\SD{0.012} & 0.4694 & 0.5823\SD{0.012} & 0.4693 \\
\plattt          & 0.5734\SD{0.016} & 0.4811 & 0.5755\SD{0.014} & 0.4800 \\
\rescal          & 0.6121\SD{0.005} & 0.4847 & 0.6134\SD{0.006} & 0.4899 \\
\slc (static+OP) & 0.6025\SD{0.007} & 0.5099 & 0.6019\SD{0.006} & 0.5463 \\
\slc+Iso         & 0.6038\SD{0.011} & 0.4672 & 0.6007\SD{0.013} & 0.4678 \\
\bottomrule
\end{tabular}
\end{table}

\base and \platt produce identical AUC (confirming Lemma~1).
\plattt \emph{decreases} AUC ($-0.8\pp$ on bbA), showing that time-only conditioning without route identity is harmful.
Per-route methods (\rescal, \slc) recover $+2$--$3\pp$ AUC headroom, so the positive-control conclusion is structural rather than method-specific.
\rescal achieves the highest AUC, while \slc is slightly more conservative in this dense regime; this matches the main-text density analysis, where shrinkage helps most when per-item calibration data are scarce and can add bias when route-level data are abundant.
\slc+Iso achieves the best NLL (0.467--0.468) while preserving the route-aware AUC gain.

\subsection{MovieLens-1M (Negative Control)}

The main text reports that MovieLens-1M with a matrix-factorization backbone yields $\Delta\AUC\approx 0$.
Table~\ref{tab:ml1m} confirms this.

We use MovieLens-1M ({$\sim$}1M ratings, 3706 movies, 6040 users) as a binary classification task, with ratings $\ge 4$ treated as positive.
The two backbones are \textsf{MF} (matrix factorization with per-item embeddings) and \textsf{NCF} (neural collaborative filtering with per-item embeddings).
We use a temporal split by timestamp and report means over 3~seeds.

\begin{table}[h]
\caption{MovieLens-1M negative control (3~seeds, mean). When the backbone already models per-item effects, post-hoc per-item correction provides negligible or no benefit.}\label{tab:ml1m}
\centering
\small
\begin{tabular}{l cc cc}
\toprule
 & \multicolumn{2}{c}{MF backbone} & \multicolumn{2}{c}{NCF backbone} \\
\cmidrule(lr){2-3} \cmidrule(lr){4-5}
Method & AUC & NLL & AUC & NLL \\
\midrule
\base    & 0.7949 & 0.5483 & 0.7885 & 0.5682 \\
\platt   & 0.7949 & 0.5469 & 0.7885 & 0.5548 \\
\rescal  & 0.7949 & 0.5481 & 0.7882 & 0.5667 \\
\slc     & 0.7951 & 0.5469 & 0.7843 & 0.5599 \\
\naive   & 0.7945 & 0.5507 & 0.7572 & 0.6479 \\
\bottomrule
\end{tabular}
\end{table}

On the MF backbone, $\Delta$AUC between \base and \slc is $+0.02\pp$---effectively zero.
On the NCF backbone, \slc \emph{decreases} AUC by $-0.42\pp$, and \naive degrades by $-3.13\pp$.
Together with flight-delay, this indicates that the effect is backbone-relative rather than domain-specific: when the backbone already incorporates item effects via embeddings, post-hoc per-item correction adds noise without recovering headroom.
The MovieLens result validates \slc's applicability boundary and demonstrates that the method is appropriately conservative (does not inflate metrics artificially).

%% file: main.bbl
\begin{thebibliography}{10}
\providecommand{\url}[1]{\texttt{#1}}
\providecommand{\urlprefix}{URL }
\providecommand{\doi}[1]{https://doi.org/#1}

\bibitem{assistments2017dataset}
{ASSISTments}: {ASSI}stments 2017 data mining dataset.
  \url{https://sites.google.com/view/assistmentsdatamining/dataset} (2017),
  accessed: 2026-02-23

\bibitem{breslow1993glmm}
Breslow, N.E., Clayton, D.G.: Approximate inference in generalized linear mixed
  models. Journal of the American Statistical Association  \textbf{88}(421),
  9--25 (1993)

\bibitem{corbett1995kt}
Corbett, A.T., Anderson, J.R.: Knowledge tracing: Modeling the acquisition of
  procedural knowledge. User Modeling and User-Adapted Interaction
  \textbf{4}(4),  253--278 (1995)

\bibitem{durbin2012time}
Durbin, J., Koopman, S.J.: Time Series Analysis by State Space Methods. Oxford
  University Press, 2nd edn. (2012)

\bibitem{efron2012large_scale}
Efron, B.: Large-Scale Inference: Empirical {B}ayes Methods for Estimation,
  Testing, and Prediction. Cambridge University Press (2010)

\bibitem{efron1973stein}
Efron, B., Morris, C.: Stein's estimation rule and its competitors---an
  empirical {B}ayes approach. Journal of the American Statistical Association
  \textbf{68}(341),  117--130 (1973)

\bibitem{fahrmeir1992kalman}
Fahrmeir, L.: Posterior mode estimation by extended {K}alman filtering for
  multivariate dynamic generalized linear models. Journal of the American
  Statistical Association  \textbf{87}(418),  501--509 (1992)

\bibitem{fawcett2006roc}
Fawcett, T.: An introduction to {ROC} analysis. Pattern Recognition Letters
  \textbf{27}(8),  861--874 (2006)

\bibitem{feng2009assistments}
Feng, M., Heffernan, N.T., Koedinger, K.R.: Addressing the assessment challenge
  with an online system that tutors as it assesses. User Modeling and
  User-Adapted Interaction  \textbf{19},  243--266 (2009)

\bibitem{frenkel2021cts}
Frenkel, L., Goldberger, J.: Network calibration by class-based temperature
  scaling. In: Proceedings of the European Signal Processing Conference
  (EUSIPCO) (2021)

\bibitem{ghosh2020akt}
Ghosh, A., Heffernan, N., Lan, A.S.: Context-aware attentive knowledge tracing.
  In: Proceedings of the 26th ACM SIGKDD International Conference on Knowledge
  Discovery \& Data Mining. pp. 2330--2340 (2020)

\bibitem{gruber2022better}
Gruber, S.G., Buettner, F.: Better uncertainty calibration via proper scores
  for classification and beyond. In: Advances in Neural Information Processing
  Systems (NeurIPS). vol.~35 (2022)

\bibitem{guo2017calibration}
Guo, C., Pleiss, G., Sun, Y., Weinberger, K.Q.: On calibration of modern neural
  networks. In: Proceedings of the 34th International Conference on Machine
  Learning (ICML). pp. 1321--1330 (2017)

\bibitem{hanley1982roc}
Hanley, J.A., McNeil, B.J.: The meaning and use of the area under a receiver
  operating characteristic ({ROC}) curve. Radiology  \textbf{143}(1),  29--36
  (1982)

\bibitem{harper2015movielens}
Harper, F.M., Konstan, J.A.: The {MovieLens} datasets: History and context. ACM
  Transactions on Interactive Intelligent Systems  \textbf{5}(4),  Article 19
  (2015)

\bibitem{he2017ncf}
He, X., Liao, L., Zhang, H., Nie, L., Hu, X., Chua, T.S.: Neural collaborative
  filtering. In: Proceedings of the 26th International Conference on World Wide
  Web (WWW). pp. 173--182 (2017)

\bibitem{vtirt2023}
Kim, Y., Sankaranarayanan, S., Piech, C., Thille, C.: Variational temporal
  {IRT}: Fast, accurate, and explainable inference of dynamic learner
  proficiency. In: Proceedings of the 16th International Conference on
  Educational Data Mining (EDM) (2023)

\bibitem{koren2009mf}
Koren, Y., Bell, R., Volinsky, C.: Matrix factorization techniques for
  recommender systems. Computer  \textbf{42}(8),  30--37 (2009)

\bibitem{lee2023kt_longitudinal}
Lee, M.P., Croteau, E., Gurung, A., Botelho, A.F., Heffernan, N.T.: Knowledge
  tracing over time: A longitudinal analysis. In: Proceedings of the 16th
  International Conference on Educational Data Mining (EDM) (2023)

\bibitem{lee2025concept_drift_kt}
Lee, M.P., Heffernan, N.T.: Concept drift detection for knowledge tracing. In:
  Proceedings of the 18th International Conference on Educational Data Mining
  (EDM), Doctoral Consortium (2025)

\bibitem{menon2021logitadj}
Menon, A.K., Jayasumana, S., Rawat, A.S., Jain, H., Veit, A., Kumar, S.:
  Long-tail learning via logit adjustment. In: Proceedings of the 9th
  International Conference on Learning Representations (ICLR) (2021)

\bibitem{pan2020field}
Pan, F., Ao, X., Tang, P., Lu, M., Liu, D., Xiao, L., He, Q.: Field-aware
  calibration: A simple and empirically strong method for reliable
  probabilistic predictions. In: Proceedings of The Web Conference (WWW). pp.
  729--739 (2020)

\bibitem{pandey2019sakt}
Pandey, S., Karypis, G.: A self-attentive model for knowledge tracing. In:
  Proceedings of the 12th International Conference on Educational Data Mining
  (EDM) (2019)

\bibitem{piech2015dkt}
Piech, C., Bassen, J., Huang, J., Ganguli, S., Sahami, M., Guibas, L.J.,
  Sohl-Dickstein, J.: Deep knowledge tracing. In: Advances in Neural
  Information Processing Systems (NeurIPS). vol.~28 (2015)

\bibitem{platt1999probabilistic}
Platt, J.C.: Probabilistic outputs for support vector machines and comparisons
  to regularized likelihood methods. In: Advances in Large Margin Classifiers.
  pp. 61--74. MIT Press (1999)

\bibitem{shen2021lpkt}
Shen, S., Liu, Q., Chen, E., Huang, Z., Huang, W., Yin, Y., Su, Y., Wang, S.:
  Learning process-consistent knowledge tracing. In: Proceedings of the 27th
  ACM SIGKDD Conference on Knowledge Discovery \& Data Mining. pp. 1452--1460
  (2021)

\bibitem{stamper2016kddcup}
Stamper, J.C., Pardos, Z.A.: The 2010 {KDD} cup competition dataset: Engaging
  the machine learning community in predictive learning analytics. Journal of
  Learning Analytics  \textbf{3}(2),  312--316 (2016)

\bibitem{tomani2022pts}
Tomani, C., Cremers, D., Buettner, F.: Parameterized temperature scaling for
  boosting the expressive power in post-hoc uncertainty calibration. In:
  Proceedings of the European Conference on Computer Vision (ECCV) (2022)

\bibitem{tomani2023dac}
Tomani, C., Waseda, F.K., Shen, Y., Cremers, D.: Beyond in-domain scenarios:
  Robust density-aware calibration. In: Proceedings of the 40th International
  Conference on Machine Learning (ICML). Proceedings of Machine Learning
  Research, vol.~202, pp. 34344--34368. PMLR (2023)

\bibitem{bts_ontime}
{U.S.\ Department of Transportation, Bureau of Transportation Statistics}:
  On-time: Reporting carrier on-time performance (1987--present).
  \url{https://www.transtats.bts.gov/}, accessed: 2026-02-23

\bibitem{wang2021tent}
Wang, D., Shelhamer, E., Liu, S., Olshausen, B., Darrell, T.: Tent: Fully
  test-time adaptation by entropy minimization. In: Proceedings of the 9th
  International Conference on Learning Representations (ICLR) (2021)

\bibitem{wang2013dynamic_irt}
Wang, X., Berger, J.O., Burdick, D.S.: Bayesian analysis of dynamic item
  response models in educational testing. The Annals of Applied Statistics
  \textbf{7}(1),  126--153 (2013)

\bibitem{wang2021eedi}
Wang, Z., Lamb, A., Saveliev, E., Cameron, P., Zaykov, Y.,
  Hern{\'a}ndez-Lobato, J.M., Turner, R.E., Baraniuk, R.G., Barton, C.,
  Peyton~Jones, S., Woodhead, S., Zhang, C.: Results and insights from
  diagnostic questions: The {NeurIPS} 2020 education challenge. In: NeurIPS
  2020 Competition and Demonstration Track. Proceedings of Machine Learning
  Research, vol.~133, pp. 191--205. PMLR (2021)

\bibitem{zadrozny2001obtaining}
Zadrozny, B., Elkan, C.: Obtaining calibrated probability estimates from
  decision trees and naive {B}ayesian classifiers. In: Proceedings of the 18th
  International Conference on Machine Learning (ICML). pp. 609--616 (2001)

\bibitem{zadrozny2002transforming}
Zadrozny, B., Elkan, C.: Transforming classifier scores into accurate
  multiclass probability estimates. In: Proceedings of the 8th ACM SIGKDD
  International Conference on Knowledge Discovery and Data Mining. pp. 694--699
  (2002)

\bibitem{zhang2017dkvmn}
Zhang, J., Shi, X., King, I., Yeung, D.Y.: Dynamic key-value memory networks
  for knowledge tracing. In: Proceedings of the 26th International Conference
  on World Wide Web (WWW). pp. 765--774 (2017)

\end{thebibliography}
